\documentclass{article}

\usepackage{amsmath,amsfonts,amssymb,enumerate,amsthm}
\usepackage{bbm}
\usepackage{latexsym}
\usepackage{graphicx}
\usepackage{epstopdf}
\usepackage[margin=0.5cm]{subcaption}
\usepackage{hyperref}
\usepackage{tikz, pgfplots}
\usepackage[margin=1in]{geometry}
\usepackage{epstopdf}
\usepackage{algpseudocode,algorithm}

\newtheorem{theorem}{Theorem}
\newcommand{\argmin}{\operatornamewithlimits{argmin}}

\begin{document}
\title{Reservoir computing for spatiotemporal signal classification without trained output weights \thanks{This work was cleared for public release by Wright Patterson Air Force Base Public Affairs on 11 Apr 2016.  Case Number: 88ABW-2016-1812.}}

\author{Ashley Prater \\Air Force Research Laboratory, Information Directorate, Rome NY USA \thanks{ashley.prater.3@us.af.mil}}

\maketitle

\begin{abstract}
Reservoir computing is a recently introduced machine learning paradigm that has been shown to be well-suited for the processing of spatiotemporal data.  Rather than training the network node connections and weights via backpropagation in traditional recurrent neural networks, reservoirs instead have fixed connections and weights among the `hidden layer' nodes, and traditionally only the weights to the output layer of neurons are trained using linear regression.  We claim that for signal classification tasks one may forgo the weight training step entirely and instead use a simple supervised clustering method based upon principal components of norms of reservoir states.  The proposed method is mathematically analyzed and explored through numerical experiments on real-world data.  The examples demonstrate that the proposed may outperform the traditional trained output weight approach in terms of classification accuracy and sensitivity to reservoir parameters.
\end{abstract}


\thispagestyle{empty}

\section{Introduction}
Reservoir computing is a recently developed bio-inspired machine learning paradigm for the processing of spatiotemporal data~\cite{jaeger, maasrealtime}.  In the language of neural networks, a reservoir is collection of hidden layer nodes with nonlinear recurrent dynamics, where the nodes are sparsely connected with fixed weights that are not trained to fit specific data.  Because the weights are fixed, using a reservoir requires only a simple initialization step, as opposed to more traditional recurrent neural networks whose weights and connections must be learned in a tedious backpropagation training step~\cite{werbos}.  The property of fixing the reservoir connections and weights has many benefits, including ease of initialization, along with having the ability to quickly adapt to new data and applications.

Reservoirs, like all recurrent neural networks, are based on the premise that the state of the reservoir at a particular time should depend on the current value of the input signal, along with recent inputs and reservoir states.  To be an effective method for computation, a reservoir should map input data into a sufficiently high-dimensional space.  It is desirable for a reservoir to operate `at the edge of chaos'~\cite{bertschinger}, so dissimilar inputs are sufficiently separated in the reservoir node states, yet inputs with only small perturbation-like differences do not stray too far apart.  Reservoir dynamics demonstrate long short-term memory~\cite{jaeger_lstm}, so any individual point-wise errors in a signal will not corrupt the entire reservoir response.  

Two types of reservoirs that have emerged in literature include echo state networks (ESNs) and time-delay reservoirs (TDRs).  An ESN uses randomly, yet sparsely, connected nodes with randomly assigned fixed weights~\cite{jaeger, maasrealtime, goodman}.  A TDR uses a cyclic topology, where each node provides data to exactly one other node, and has fixed, non-random weights~\cite{grigoryeva, ortin, paquotopto}.  
The output layer of both ESN and TDR-type reservoirs traditionally use linear output weights, trained on a labeled dataset using least squares or ridge regression\cite{jaeger, goudarzi, luko}.  This method has an easy training phase, and is computationally cheap to use in the testing phase.  However, it can be sensitive to reservoir parameters and dataset characteristics and prone to overfitting.  If the training dataset has large intra-class variation, or if the classes are not well-separated, then it may be difficult to find a collection of weights to discriminate the classes well.  

In this research, a simple supervised clustering method based on principal components is proposed for use in classification tasks using ESNs and TDRs.    The method used is based upon comparing the norm of a reservoir response of a test signal against the principal components of the norms of reservoir states for classes of labeled training data.  The clustering method has slightly higher computational complexity than using trained output weights to classify new input signals, however it may achieve higher classification accuracy while being less sensitive to reservoir type, size, and feedback strength.  We present a rigorous analysis of the clustering method, including two theoreoms characterizing the upper bound of the difference in reservoir responses for two input signals, with the upper bound in terms of the input signals, the reservoir type, and the user-generated parameters.  Moreover, we explore the difference in performance of the two methods through numerical simulations performed using a real-world dataset for both ESNs and TDRs for various reservoir parameters.  In every simulation, the clustering approach outperforms the trained output weights in terms of both accuracy and CPU time required to classify test signals.

The following notation is used in this work.  For a collection of signals $\{u\}$, the $j^\text{th}$ element in the collection is denoted by $u^{(j)}$.  Training sets are partitioned into $K$ classes.  Let $\mathcal{C}_k$ be the collection of indices of signals in the $k^\text{th}$ class.  That is, $u^{(j)}$ is in the $k^\text{th}$ class iff $j\in \mathcal{C}_k$.   For a vector $v$, the $\ell_2$ norm is given by $\|v\|_2 = (\sum_j v^2_j)^{1/2}$.  For a matrix $A$, $\rho(A)$ is the spectral radius, i.e.\ the largest absolute value of an eigenvalue of $A$.  We use $\mathcal{O}(\cdot)$ with the standard `big O' meaning, that $f(x) = \mathcal{O}(g(x))$ if there exists $M>0$ and $x_0\in\mathbb{R}$ such that $|f(x)|\leq M|g(x)|$ for all $x\geq x_0$.

\section{Reservoir Computing Models for Classification}

Suppose $u\in\mathbb{R}^T$ is an input signal of length $T$, possibly after the application of a multiplexing mask, and say $u(t)$ is the value of $u$ at time $t$. The values of the $N$ reservoir nodes at time $t$ are called the \emph{reservoir states} and are denoted by the vectors $X(t)\in\mathbb{R}^N$, one vector for each $t$. The $n^\text{th}$ entry of these vectors, $X_n(t)$ denote the state of the $n^\text{th}$ reservoir node at time $t$.  The dynamics of the ESN and TDR architectures are described by the following models:
\begin{align}
	\label{eq:1}&\text{ESN:   }   X(t) = f( W_\text{in} u(t) + W_\text{res} X(t-1)) \\
	\label{eq:2}&\text{TDR:   }   X_n(t) = \begin{cases} f(\alpha u(t) + \beta X_{N-1}(t-1)), \; &\text{ if } n = 0\\
		X_{n-1}(t-1), &\text{ if } n \in \{1,2,\ldots,N-1\}
	\end{cases}
\end{align}

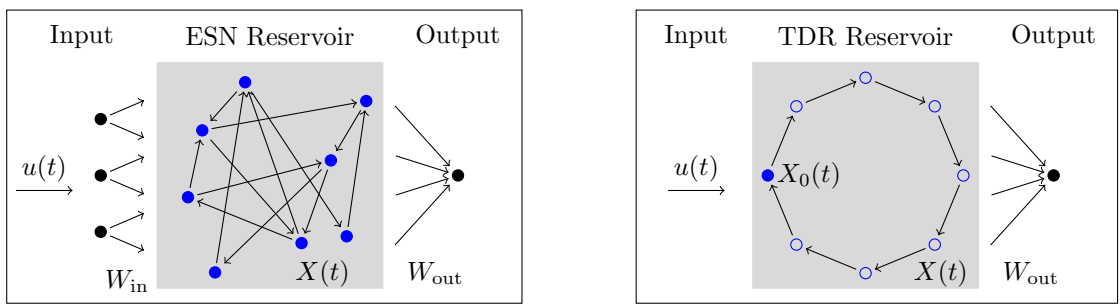
\begin{figure}[bthp]
\centering
\begin{tikzpicture}

	\node at (-1,3.35) {Input};
	\node (In1) at (-0.75,0.75) {};
	\node (In2) at (-0.75,1.5) {};
	\node (In3) at (-0.75,2.25) {};

	\node (In1b) at (-.05,.45) {};
	\node (In1t) at (-.05,1.05) {};
	\node (In2b) at (-.05,1.2) {};
	\node (In2t) at (-.05,1.8) {};
	\node (In3b) at (-.05,1.95) {};
	\node (In3t) at (-.05,2.55) {}; 

	\draw [black,fill] (In1) circle (0.075);
	\draw [black,fill] (In2) circle (0.075);
	\draw [black,fill] (In3) circle (0.075);

	\draw [->] (In1) edge (In1b);
	\draw [->] (In1) edge (In1t);
	\draw [->] (In2) edge (In2b);
	\draw [->] (In2) edge (In2t);
	\draw [->] (In3) edge (In3b);
	\draw [->] (In3) edge (In3t);
	
	\node (u) at (-1.5,1.6) {$u(t)$};
	\node (ul) at (-2,1.3) {};
	\node (ur) at (-1,1.3) {};
	\draw [->] (ul) edge (ur);

	\node at (-0.4,0.1) {$W_\text{in}$};

	\node at (4,3.35) {Output};
	\node (O1) at (4,1.5) {};

	\node (O1b) at (3.05,.45) {};
	\node (O1t) at (3.05,1.05) {};
	\node (O2b) at (3.05,1.2) {};
	\node (O2t) at (3.05,1.8) {};
	\node (O3b) at (3.05,1.95) {};
	\node (O3t) at (3.05,2.55) {}; 

	\draw [black,fill] (O1) circle (0.075);

	\draw [->] (O1b) edge (O1);
	\draw [->] (O2b) edge (O1);
	\draw [->] (O2t) edge (O1);
	\draw [->] (O3t) edge (O1);

	\node at (3.7,0.2) {$W_\text{out}$};

	\draw [black, thin] (-2,-.2) rectangle (4.85,3.7);

	\draw [gray!30,fill] (0,0) rectangle (3,3);
	\node at (1.5,3.35) {ESN Reservoir};

	\node (A) at (0.41, 1.21) {};
	\node (B) at (0.60, 2.10) {};
	\node (C) at (1.17, 2.74) {};
	\node (D) at (1.92, 0.60) {};
	\node (E) at (0.77, 0.21) {};
	\node (F) at (2.31, 1.70) {};
	\node (G) at (2.52, 0.69) {};
	\node (H) at (2.78, 2.49) {};

	\node at (2.2,.2) {$X(t)$};

	\draw [blue, fill] (A) circle (0.075);
	\draw [blue, fill] (B) circle (0.075);
	\draw [blue, fill] (C) circle (0.075);
	\draw [blue, fill] (D) circle (0.075);
	\draw [blue, fill] (E) circle (0.075);
	\draw [blue, fill] (F) circle (0.075);
	\draw [blue, fill] (G) circle (0.075);
	\draw [blue, fill] (H) circle (0.075);

	\draw [->] (A) edge (B);
	\draw [->] (A) edge (F);
	\draw [->] (B) edge (D);
	\draw [->] (B) edge (H);
	\draw [->] (C) edge (B);
	\draw [->] (C) edge (G);
	\draw [->] (D) edge (A);
	\draw [->] (D) edge (C);	
	\draw [->] (E) edge (C);
	\draw [->] (F) edge (D);
	\draw [->] (F) edge (E);
	\draw [->] (G) edge (H);
	\draw [->] (H) edge (F);
\end{tikzpicture}
\hspace{0.5 in}
\begin{tikzpicture}
	\node at (-0.75,3.35) {Input};
	
	\node (u) at (-0.75,1.6) {$u(t)$};
	\node (ul) at (-1.25,1.3) {};
	\node (ur) at (-0.25,1.3) {};
	\draw [->] (ul) edge (ur);

	\node at (4,3.35) {Output};
	\node (O1) at (4,1.5) {};

	\node (O1b) at (3.05,.45) {};
	\node (O1t) at (3.05,1.05) {};
	\node (O2b) at (3.05,1.2) {};
	\node (O2t) at (3.05,1.8) {};
	\node (O3b) at (3.05,1.95) {};
	\node (O3t) at (3.05,2.55) {}; 

	\draw [black,fill] (O1) circle (0.075);

	\draw [->] (O1b) edge (O1);
	\draw [->] (O2b) edge (O1);
	\draw [->] (O2t) edge (O1);
	\draw [->] (O3t) edge (O1);

	\node at (3.7,0.2) {$W_\text{out}$};

	\draw [gray!30,fill] (0,0) rectangle (3,3);
	\node at (1.5,3.35) {TDR Reservoir};

	\node (A) at (2.80, 1.50) {};
	\node (B) at (2.42, 2.42) {};
	\node (C) at (1.50, 2.80) {};
	\node (D) at (0.58, 2.42) {};
	\node (E) at (0.20, 1.50) {};
	\node (F) at (0.58, 0.58) {};
	\node (G) at (1.50, 0.20) {};
	\node (H) at (2.42, 0.58) {};

	\draw [blue] (A) circle (0.075);
	\draw [blue] (B) circle (0.075);
	\draw [blue] (C) circle (0.075);
	\draw [blue] (D) circle (0.075);
	\draw [blue,fill] (E) circle (0.075);
	\draw [blue] (F) circle (0.075);
	\draw [blue] (G) circle (0.075);
	\draw [blue] (H) circle (0.075);
	\node at (2.5,.2) {$X(t)$};
	\draw (E) node [right] {$X_0(t)$};

	\draw [->] (A) edge (H);
	\draw [->] (B) edge (A);
	\draw [->] (C) edge (B);
	\draw [->] (D) edge (C);
	\draw [->] (E) edge (D);
	\draw [->] (F) edge (E);
	\draw [->] (G) edge (F);
	\draw [->] (H) edge (G);

	\draw [black, thin] (-1.55,-.2) rectangle (4.85,3.7);

\end{tikzpicture}
\caption{Representations of two architectural variants of reservoirs with output weights, the echo state network (left) and time delay reservoir (right).}
\label{fig:Figure1}
\end{figure}

For ease of notation, suppose each reservoir type has $N$ nodes.  In the ESN topology, the vector $W_\text{in}\in\mathbb{R}^N$ weights the input signal feeding into the nodes, while the matrix $W_\text{res}\in\mathbb{R}^{N\times N}$ determines the fixed connections and weights among the nodes.  That is, node $m$ feeds into node $n$ weighted by the $(n,m)^\text{th}$ entry of $W_\text{res}$ in the ESN model.  The TDR has $N-1$ virtual nodes, corresponding to $n=1,2,\ldots,N-1$, and one physical node for $n=0$.  The parameter $\alpha$ is the input gain, and $\beta$ is the attenuation value.  Notice in the TDR the node values are simply passed along the reservoir unchanged except at the physical node.  Models of the ESN and TDR reservoirs are shown in Figure~\ref{fig:Figure1}.

The function $f$ in Equations~\eqref{eq:1} and~\eqref{eq:2} is a nonlinear activation function.  Typical choices for $f$ include sinusoidal, logistic, sigmoidal, and piecewise linear functions.

Time-multiplexing the inputs is a common preprocessing step in time-delay reservoir systems~\cite{grigoryeva, ortin, paquotopto,appletantthesis, appletantoptimized}.  The multiplexing mask is applied as follows.  Suppose the raw inputs are $z_1, z_2, z_3,\ldots$, and consider a mask $m$ of length $L$.  Then the multiplexed input $u$ is defined via $u((k-1)L+1:kL) = z_km$, that is each raw input $z_k$ is multiplied by the vector $m$ and concatenated to form the multiplexed input.  The purpose of the multiplexing mask in TDRs is several-fold.  A non-constant mask helps to increase the dimensionality of the reservoir, yielding richer dynamics~\cite{paquotopto,appletantoptimized}.  Furthermore, since the inputs are all passed to the reservoir only via the head node, the mask allows several virtual nodes to process values from a single raw input vector at once, as random ESNs do by design~\cite{appletantthesis, appletantoptimized}.    An additional benefit is that it helps to `slow down' TDRs, many of which are physically implemented as optical devices that would process the raw data much faster than one can sample the outputs~\cite{ortin, paquotopto, duport}.

Two approaches for interpreting the reservoir outputs for supervised classification tasks are described in the subsections below.  The first describes the traditional approach using trained output weights, and the second describes a method of clustering the reservoir node states.  

\subsection{Trained linear output weights}
A classical method to interpret the results of a reservoir is to train a collection of output weight matrices $W_\text{out}(t) \in \mathbb{R}^{K\times N}$ at each time $t$ of interest~\cite{jaeger, goudarzi, luko, luko2} that map reservoir states close to an appropriate `indicator' vector.  That is, choose a collection of times of interest $\Omega \subseteq \{1,2,\ldots,T\}$, and let $X^{(j)}(t)\in\mathbb{R}^M$ denote the reservoir nodes at time $t$ driven by the $j^\text{th}$ element in the training dataset using either Equation~\eqref{eq:1} or~\eqref{eq:2}.   Then the collection of output weights at time $t$ is
\begin{equation}\label{eq:Wout t}
W_\text{out}(t) = \argmin_{W\in\mathbb{R}^{K\times N}} \left\{ \sum_{j\in\mathrm{Tr}} \|d_j - WX^{(j)}(t) \|_2^2 + \lambda\|W\|_2^2 \right\},
\end{equation}
where each $d_j(t)\in\mathbb{R}^K$ is an indicator vector, having all zeros entries except for a $1$ in the $k^\text{th}$ position if $j\in\mathcal{C}_k$.

Equipped with the collection of output weights, a test pattern with reservoir node states $X(t) \in \mathbb{R}^{N}$ is determined to belong to the $k^\text{th}$ class if the $K$-vector 
\begin{equation}\label{eq:D test merged}
	D = \sum_{t\in\Omega} \omega_t W_\text{out}(t)X(t)
\end{equation}
has maximal element in the $k^\text{th}$ row.  Typically the classification weights $\omega_t$ are chosen to be 1.

\begin{algorithm}[htb!]\caption{(To classify a signal using trained linear output weights)}\label{alg:Wout t}
 \begin{algorithmic}[htb!]
	\State \textbf{Initialization:} 
	Input the fixed parameters $\lambda, \Omega$. 

	\State \textbf{Training:} Generate the vectors $X^{(j)}(t)\in\mathbb{R}^N$ for each $j$ and $t\in\Omega$, using Equation~\eqref{eq:1} or~\eqref{eq:2}, then find the collections $\{W_\text{out}(t): t\in\Omega\}$ as in Equation~\eqref{eq:Wout t}.

	\State \textbf{Testing:} Let $u\in\mathbb{R}^T$ be a new test pattern.
	\begin{enumerate}
		\item Compute the corresponding reservoir nodes $\{X(t) \in\mathbb{R}^N: t\in\Omega\}$ using Equation~\eqref{eq:1} or~\eqref{eq:2}.  
		\item Compute the vector $D$ as in Equation~\eqref{eq:D test merged}.
		\item Say $u$ is in the $k^\text{th}$ class if $D(k) \geq D(\ell)$ for all indices $\ell \in \{1:K\}$.
	\end{enumerate}
 \end{algorithmic}
\end{algorithm}

The computational cost of determining the class of a new pattern using the the `Testing' phase of Algorithm~\ref{alg:Wout t} is determined as follows.  Assume that the matrices $W_\text{out}(t)$ are given, and do not include its derivation in the cost evaluation.   To drive the reservoir and find the nodes $X(t)$ of a new test pattern requires $\mathcal{O}(N^2T)$ multiplications using the ESN dynamics in Equation~\eqref{eq:1}, or $\mathcal{O}(T)$ multiplications using the TDR dynamics in Equation~\eqref{eq:2}.  Although only the reservoir node values at times $t\in\Omega$ are of interest, one must drive the reservoir using the full set of times.  To find the vector $D$ requires $\mathcal{O}(KN|\Omega|)$ multiplications, and finally $\mathcal{O}(K)$ comparisons are needed to determine the maximal element.  Overall, this leads to a complexity of $\mathcal{O}(N^2T + KN|\Omega|)$ when using ESN-type reservoirs and a complexity of $\mathcal{O}(T + KN|\Omega|)$ when using TDR-type reservoirs.

\subsection{Classification via Clustering with Principal Components}

The underlying idea for the training method~\eqref{eq:Wout t} is that similar inputs to the reservoir produce similar outputs, even after the non-linear high-dimensional processing is applied.  Under this assumption, it is feasible that one could classify data using a clustering method without the use of the trained output weights.  Therefore, we propose the following method using the principal components of norms of reservoir responses to perform classification.  Let $\mathcal{C}_k$ be the collection of indices of training patterns that belong to the $k^\text{th}$ class.  Find the reservoir states $X^{(j)}(t)\in\mathbb{R}^N$ for all $j\in\mathcal{C}_k$, $t\in\Omega$, and $k$, and compute the vectors $b_j\in\mathbb{R}^{|\Omega|}$, where 
\begin{equation}\label{eq:b}
	b_j(i) = \left\| X^{(j)}(t_i) \right\|_2^2, \quad t_i\in\Omega.
\end{equation}
For each class, concatenate the vectors $b_j$ to form matrices $B_k \in \mathbb{R}^{|\Omega|\times |\mathcal{C}_k|}$.  Since the input training patterns belong to the same class, the columns of each $D_k$ should exhibit similar characteristics.  Suppose $U_k\in\mathbb{R}^{|\Omega|\times R}$ is the matrix of the first $R$ principal components of $B_k$.  For any new test patterns with corresponding vector $b$, one can say that the pattern belongs to the $k^\text{th}$ class if $U_k$ describes $b$ well, i.e.\ if 
\begin{equation*}
	\left\| \left( I - U_kU_k^*\right)b \right\| \leq \left\| \left( I - U_\ell U_\ell^*\right)b\right\|, \quad \forall \ell.
\end{equation*}

\begin{algorithm}[htb!]\caption{(To classify a signal using clustering via principal components)}\label{alg:cluster}
\begin{algorithmic}[htb!]
	\State \textbf{Initialization}: Input the collection of times $\Omega$ and the number of principal components to consider $R$.

	\State \textbf{Training}: 
	\begin{enumerate}
		\item Generate the vectors $X^{(j)}(t)\in\mathbb{R}^N$ for all $j\in\mathcal{C}_k$, $t\in\Omega$, and each $k$ using either Equation~\eqref{eq:1} or~\eqref{eq:2}.
		\item Compute the matrices $B_k$ with columns as in Equation~\eqref{eq:b}.
		\item Compute $U_k$, the first $R$ principal components of $B_k$.
	\end{enumerate}

	\State \textbf{Testing}: Let $u \in\mathbb{R}^T$ be a new test pattern.
	\begin{enumerate} 
		\item Generate the reservoir states $X(t)\in\mathbb{R}^N$ for each $t$ according to Equation~\eqref{eq:1} or~\eqref{eq:2}.
		\item Compute the vector $b\in\mathbb{R}^{|\Omega|}$ according to Equation~\eqref{eq:b}.
		\item For each $k$, compute $d_k = \|(I_{|\Omega|} - U_kU_k^*)b\|_2^2$.
		\item Say $u$ belongs to the $k^\text{th}$ class if $d_k\leq k_\ell$ for all $\ell$.
	\end{enumerate}
\end{algorithmic}
\end{algorithm}

The computational cost of determining the class of a new pattern using the `Testing' phase of Algorithm~\ref{alg:cluster} is determined as follows.  Assume that the matrices $I - U_kU_k^* \in \mathbb{R}^{|\Omega|\times |\Omega|}$ are precomputed during the training phase.  As in Algorithm~\ref{alg:Wout t}, the cost to drive the reservoir and find the nodes $X(t)$ requires $\mathcal{O}(N^2T)$ multiplications using the ESN, or $\mathcal{O}(T)$ multiplications using the TDR.  To compute the vector $b$ in Step~2 requires $\mathcal{O}(N|\Omega|)$ multiplications, and to compute the values $\{d_k\}$ in Step~3 requires $\mathcal{O}(K|\Omega|^2)$ multiplications.  Finally, to determine the class of $u$ in Step~4 requires $K$ comparisons.  Overall, this leads to a complexity of $\mathcal{O}(N^2T + N|\Omega| + K|\Omega|^2)$ for ESN-type reservoirs, and a complexity of $\mathcal{O}(T + N|\Omega| + K|\Omega|^2)$ for TDR-type reservoirs.  Since the parameters can vary in magnitude, the dominant term in the complexity depends on the particular set-up used.

\section{Analysis of Reservoir Behavior}
The clustering method proposed in Algorithm~\ref{alg:cluster} will be more accurate if small variations in the input signals lead to bounded differences in reservoir states, while large discrepancies in inputs are mapped farther apart.  To confidently use this approach, we must characterize reservoir responses for similar inputs.

Several studies of reservoir performance based on the type of reservoir architecture, chosen parameters, as well as the characteristics of the input data have been performed, with evidence that some combinations of the aforementioned factors can seriously degrade performance~\cite{maasrealtime, grigoryeva, goudarzi, paquotdambre}.  However, the metrics used in the reservoir computing literature tend to be only experimentally investigated.   To explore how well the reservoir response separates classes, the separation ratio~\cite{goodman, gibbons}, point-wise separation~\cite{maasrealtime, schrauwen2007}, and class separation~\cite{chrol-cannon} have been used.  These all measure how well a reservoir can separate inputs from distinct classes, by having distances between disparate classes large while keeping similar inputs close.  Similarly, to measure how effectively a reservoir can process a particular dataset, researchers use the universal approximation property~\cite{maasrealtime} kernel quality~\cite{chrol-cannon, leg, soriano}, reservoir capacity~\cite{dambre}, and the Echo State Property~\cite{jaeger}.  These measures and properties concern the representation of inputs within the reservoir response and the reconstructability of an input signal from reservoir states.  For robustness to noise, generalization rank~\cite{soriano} or the Lyapunov coefficient~\cite{gibbons, chrol-cannon, leg, verstraeten, schrauwen2008, nils} are considered.

Although the reservoir dynamics~\eqref{eq:1} and~\eqref{eq:2} have simple descriptions, rigorous treatment of their behavior have proven difficult, with few results so far.  In Proposition~3 of~\cite{jaeger}, the distance between two reservoir states at a given time is bounded in terms of the reservoir states at the previous timestep and the spectral radius of the reservoir weights.  Although mathematically proven, this Proposition covers only randomly connected ESNs incrementing one timestep with activation functions of the form $f(x) = \mathrm{tanh}(x)$.   Theorem~3.5 of~\cite{dimarco} bounds the distance between two output vectors of a TDR, determined using linear read-out weights, in terms of the reservoir parameters and the behavior of the input signals.   In the Theorems below, we prove upper bounds for distances between two reservoir responses to different inputs in terms of reservoir parameters and the behavior of the inputs for both ESNs and TDRs, and in more generality than the results given in \cite{jaeger} and~\cite{dimarco}.

For readability, let us first introduce some notation.  Let $u^{(j)}(t)$ denote the $j^\text{th}$ input at time $t$, with corresponding reservoir states $X^{(j)}(t)$.  Let $\delta_{i,j,t} = |u^{(i)}(t)-u^{(j)}(t)|$ be the difference between two input signals at time $t$, and let $\varepsilon_{i,j,t} = \|x^{(i)}(t)-x^{(j)}(t)\|$ be the distance between the corresponding node states at time $t$.  Suppose $\overline{\delta}_{i,j} = \sup \{ \delta_{i,j,t} : t \in \mathbb{R}\}$ is bounded for each pair $(i,j)$, and that the nonlinear activation function $f$ is Lipschitz continuous with optimal Lipschitz constant $L$.  Finally, let $[\cdot]_n$ denote a vector whose entries run over the range of the variable $n$.

\begin{theorem}\label{th:1}
Suppose the reservoir node states are determined using the ESN dynamics from Equation~\eqref{eq:1}.  If $\rho(W_\text{res})$ is the spectral radius of $W_\text{res}$, then the distance between the reservoir nodes at time $t$ corresponding to two input signals $u^{(i)}$ and $u^{(j)}$ satisfies
\begin{equation*} \varepsilon_{i,j,t} \leq L \overline{\delta}_{i,j} \|W_\text{in}\|  \frac{1-(L\rho(W_\text{res}))^{t+1}}{1-L\rho(W_\text{res})}. 
\end{equation*}
\end{theorem}

\begin{proof}
By Equation~\eqref{eq:1} and the Lipschitz continuity of $f$,
\begin{align*}
	\varepsilon_{i,j,t} &= \left\| X^{(i)}(t) - X^{(j)}(t) \right\| \\
	&= \left\| f\left(W_\text{in} u^{(i)}(t) + W_\text{res}X^{(i)}(t-1)\right) - f\left(W_\text{in} u^{(j)}(t) + W_\text{res}X^{(j)}(t-1)\right) \right\| \\
	&\leq L \left\| W_\text{in} \left[u^{(i)}(t) - u^{(j)}(t) \right] + W_\text{res}\left[ X^{(i)}(t-1)-X^{(j)}(t-1)\right] \right\| \\
	&\leq L\left\| W_\text{in} \right\| \delta_{i,j,t} + L \rho(W_\text{res}) \varepsilon_{i,j,t-1}.
\end{align*}
Since $\varepsilon_{i,j,-1} = 0$, it follows by induction that 
\[ \varepsilon_{i,j,t}\leq L\left\|W_\text{in}\right\| \sum_{r=0}^t (L\rho(W_\text{res})^r \delta_{i,j,t-r} \leq L\overline{\delta}_{i,j} \left\|W_\text{in}\right\|\frac{1 - (L\rho(W_\text{res}))^{t+1}}{1 - L\rho(W_\text{res})}. \qedhere \] 
\end{proof}

\begin{theorem}\label{th:2}
Suppose the reservoir node states are determined using the TDR dynamics from Equation~\eqref{eq:2}. Then the distance between the reservoir nodes corresponding to two input signals $u^{(i)}$ and $u^{(j)}$ satisfies
\begin{equation*} \varepsilon_{i,j,t} \leq \alpha \overline{\delta}_{i,j} L \sqrt{N}\;  \frac{ 1 - \left(\beta L\right)^{\lfloor t/N\rfloor + 1}}{1 - \beta L}. \end{equation*}
\end{theorem}

\begin{proof}
By Equation~\eqref{eq:2} and the Lipschitz continuity of $f$,
\begin{align*}
	\varepsilon_{i,j,t} &= \left\| [X_n^{(i)}(t)]_n - [X_n^{(j)}(t)]_n \right\| \\
	&= \left\| [X_0^{(i)}(t-n)]_n - [X_0^{(j)}(t-n)]_n \right\| \\
	&\leq \alpha L \left\| [u^{(i)}(t-n) - u^{(j)}(t-n)]_n \right\| + \beta L \left\| [X_{N-1}^{(i)}(t-n-1) - X_{N-1}^{(j)}(t-n-1)]_n \right\| \\
	&\leq \alpha L \overline{\delta}_{i,j} \sqrt{N} + \beta L \left\| [X_n(t-N)]_n \right\| \\
	&= \alpha L \overline{\delta}_{i,j} + \beta L \varepsilon_{i,j,t-N}.
\end{align*}
Let $r\in\{0,1,\ldots,N-1\}$ be the remainder when $t$ is divided by $N$. By induction on the inequality above,
\[ \varepsilon_{i,j,t} \leq (\beta L)^{\lfloor t/N\rfloor} \varepsilon_{i,j,r} + \alpha \overline{\delta}_{i,j} L \sum_{k=0}^{\lfloor t/N\rfloor-1} (\beta L)^k.\]
Since $r<N$, the $n^\text{th}$ reservoir node at time $r$ can be characterized by
\[ 
	X_n(r) = \begin{cases} f(\alpha u(r-n)),&\text{ if } n\leq r \\ 0,&\text{ if }n>r\end{cases},
\]
yielding $\varepsilon_{i,j,r} \leq \alpha \overline{\delta}_{i,j} L \sqrt{r+1} \leq \alpha \overline{\delta}_{i,j} L \sqrt{N}.$
Therefore 
\[ \varepsilon_{i,j,t} \leq \alpha \overline{\delta}_{i,j} L \sqrt{N} \sum_{k=0}^{\lfloor t/N\rfloor} (\beta L)^k = \alpha \overline{\delta}_{i,j}L\sqrt{N} \frac{1-(\beta L)^{\lfloor t/N\rfloor + 1}}{1 - \beta L}. \qedhere\] 
\end{proof}

Theorems~\ref{th:1} and~\ref{th:2} show that for input signals with small pointwise discrepancies and well-chosen reservoir parameters, their associated reservoir state norms cluster well.   However, the Theorems do not guarantee that very distinct inputs are mapped to dissimilar reservoir node state norms.  For this, we turn to the separation ratio, introduced in~\cite{goodman} and further explored in~\cite{gibbons}.  For completeness, we include it here, modified for both Algorithm~\ref{alg:Wout t} and Algorithm~\ref{alg:cluster}.  For Algorithm~\ref{alg:Wout t}, operator on the reservoir responses themselves, and for Algorithm~\ref{alg:cluster} consider the norms of the reservoir responses.

Define the center of mass of the reservoir states of the $K^\text{th}$ class in the training set at time $t$ as $M_k(t)$,
\[ M_k(t) = \begin{cases} \frac{1}{|\mathcal{C}_k|} \sum_{j\in \mathcal{C}_k} X^{(j)}(t), &\text{for Algorithm~\ref{alg:Wout t}},\\
	\frac{1}{|\mathcal{C}_k|} \sum_{j\in\mathcal{C}_k} \| X^{(j)}(t) \|, &\text{for Algorithm~\ref{alg:cluster}}.
	\end{cases}
\]
The inter-class distance is computed the same for both algorithms.  It is defined as the average distance between pairs of class means at each time step:
\[	d(t) = \frac{1}{K^2} \sum_{k=1}^K \sum_{\ell=1}^K\left\|  M_k(t) - M_\ell(t)\right\|, \]
The intra-class variance is the average variance within each class at each time step:
\[  	v(t) = \begin{cases}
	\frac{1}{K} \sum_{k=1}^K \frac{1}{|\mathcal{C}_k|} \sum_{j\in\mathcal{C}_k} \left\| M_k(t) - X^{(j)}(t)\right\|,& \text{for Algorithm~\ref{alg:Wout t}},\\
	\frac{1}{K} \sum_{k=1}^K \frac{1}{|\mathcal{C}_k|} \sum_{j\in\mathcal{C}_k} \left| M_k(t) - \|X^{(j)}(t)\|\right|,& \text{for Algorithm~\ref{alg:cluster}}.
	\end{cases}\]

Then the separation ratio at time $t$ is defined as
\begin{equation}\label{eq:6}	\mathrm{Sep}(t) = \frac{d(t)}{1+w(t)}. \end{equation}
The larger $\mathrm{Sep}(t)$ is, the better the separation among the classes at time $t$.

\begin{figure}[htbp]
	\centering
	\includegraphics[scale=3]{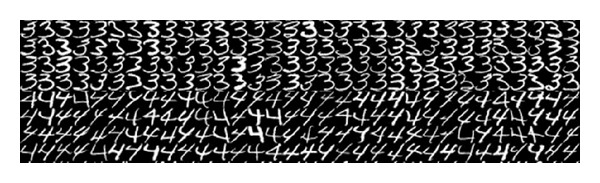}
	\caption{A subset of images from the USPS handwritten digit dataset.}
	\label{fig:2}
\end{figure}

\section{Example}

Handwritten digits are classified using the trained linear output weights in Algorithm~\ref{alg:Wout t} and the using the principal components method in Algorithm~\ref{alg:cluster}.  The data used are from the United States Postal Service (USPS) database, obtained from~\cite{roweis}. A sample of these images is shown in Figure~\ref{fig:2}.  Each image in the dataset is a $16\times 16$ pixel 8-bit grayscale image, reshaped as a 256 length column vector taking values in $[0,1]$.  The data are split in ten classes of 1100 images each, representing the digits 0 through 9.   For each simulation presented below, 400 images in each class are randomly selected to form the training set, while the remaining 7000 images (700 from each class) are used as the test set.  Although the nearby pixel behavior is not preserved in the horizontal direction by transforming each image into a column vector, the correlations are still present in the reservoir response due to the long short term memory property.

All experiments are implemented in MATLAB R2013a on a node with 2 Intel Xeon 5650 CPUs with 8 cores at 2.67 GHz with 8GB RAM.
%
\begin{figure}[hp]
	\centering
	\includegraphics[width=0.48\textwidth]{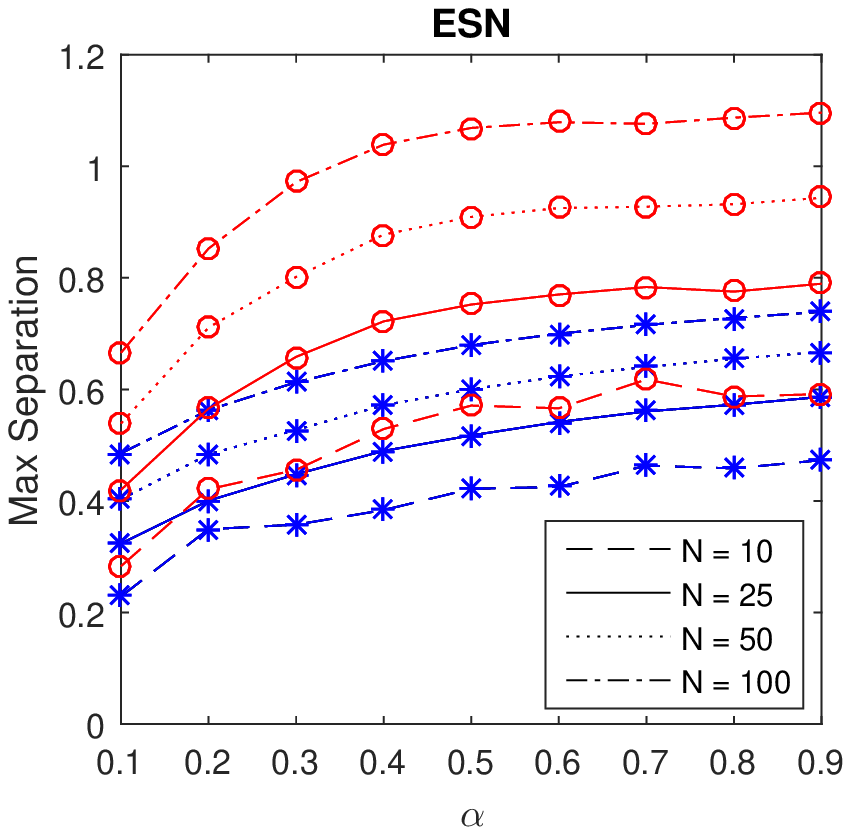}
	\hspace{.1 in}
	\includegraphics[width=0.48\textwidth]{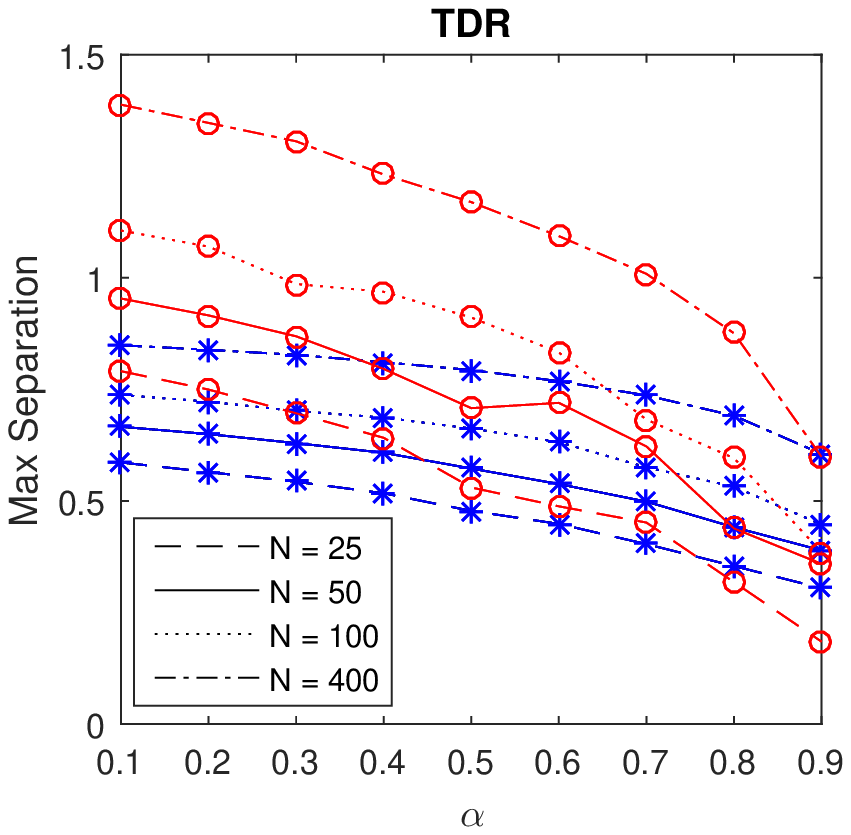}
	\caption{The maximum separation ratio from Equation~\eqref{eq:6} attained by the reservoirs on sample training sets, adapted for Algorithm~\ref{alg:Wout t} in blue `{\color{blue}$*$}' and for Algorithm~\ref{alg:cluster} in red `{\color{red}$\circ$}', for several values of $\alpha$ and $N$. Results for the ESN-style reservoirs are in the left plot, and results for TDR-style reservoirs are in the right plot.}
	\label{fig:3}
\end{figure}

\subsection{Experiment Setup}

The ESN reservoirs are set up with $N$ nodes, for $N \in\{10,25,50,100\}$.  The input weights are $W_\text{in} = \begin{bmatrix}\alpha&\alpha&\cdots&\alpha\end{bmatrix}^\top \in \mathbb{R}^N$, where $\alpha$ ranges over the set $\{0.1, 0.2, \ldots, 0.9\}$. The reservoir weights $W_\text{res}\in\mathbb{R}^{N\times N}$ are randomly chosen with $20\%$ density, and scaled so that the largest eigenvalue is $0.9999(1-\alpha)$.  No mask is used with ESN reservoirs, so $T = 256$, $\Omega = \{1,2,\ldots,256\}$, and $K=10$.

The TDR-type reservoirs are set up with $N$ nodes, for $N\in\{25,50,100,400\}$.  Again, the parameter $\alpha$ appearing in Equation~\eqref{eq:2} ranges over the set $\{0.1,0.2,\ldots,0.9\}$.  The inputs are multiplexed with a mask of length $N-1$ randomly taking values from $\{\pm 1\}$, but the reservoir is sampled only every $N-1$ time-steps.  Therefore $T = 256(N-1)$ and $\Omega = \{ r(N-1)+1: r=0,1,\ldots,255\}$ with $|\Omega|=256$.

For each simulation, 400 images from each class are randomly selected to form the training dataset, however, the same training dataset selection is used for each pair $(N,\alpha)$.  The nonlinear activation function is chosen to be $f(x) = \sin(x)$ throughout.  For the trained linear output weights, the regularization parameters $\lambda = 10^{-4}$ and $\lambda = 10^{-10}$ are used.

\begin{figure}
	\centering
	\includegraphics[width=0.45\textwidth]{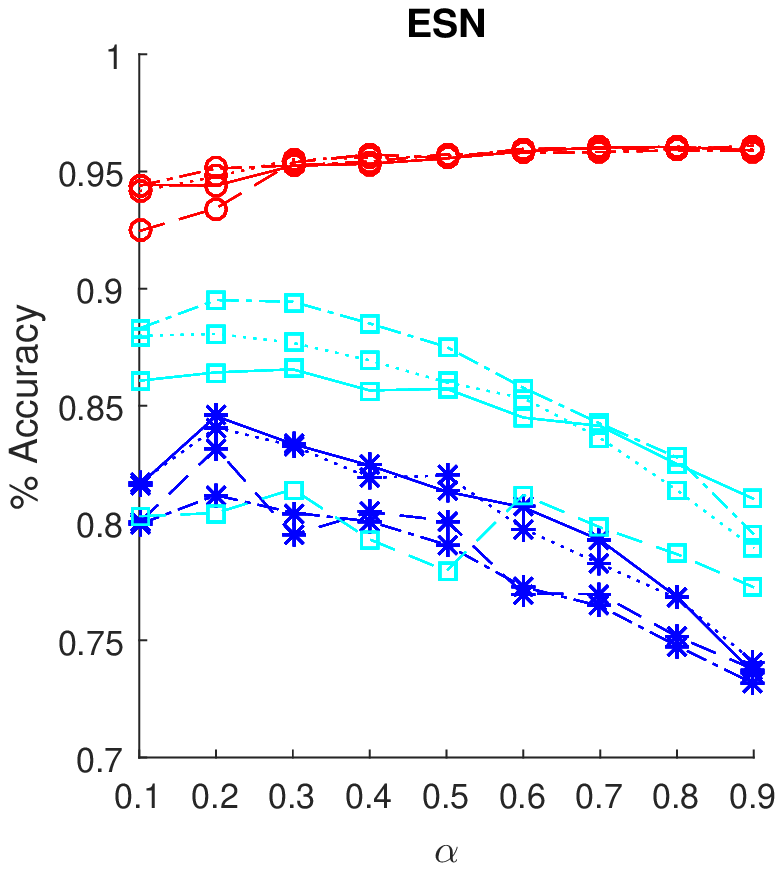}
	\includegraphics[width=0.45\textwidth]{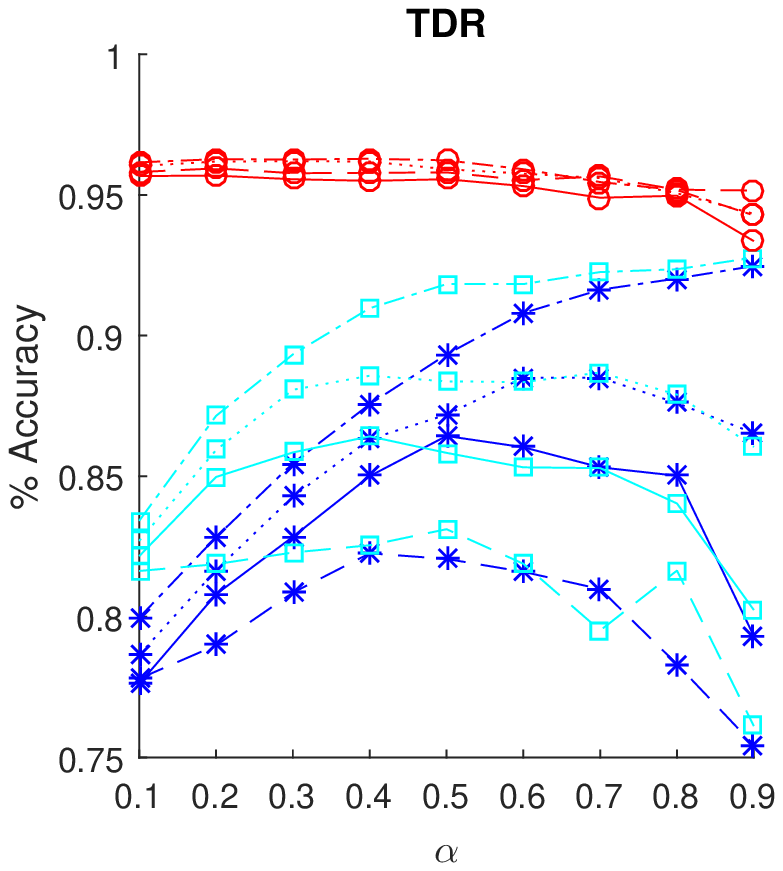}

	\includegraphics[width=0.45\textwidth]{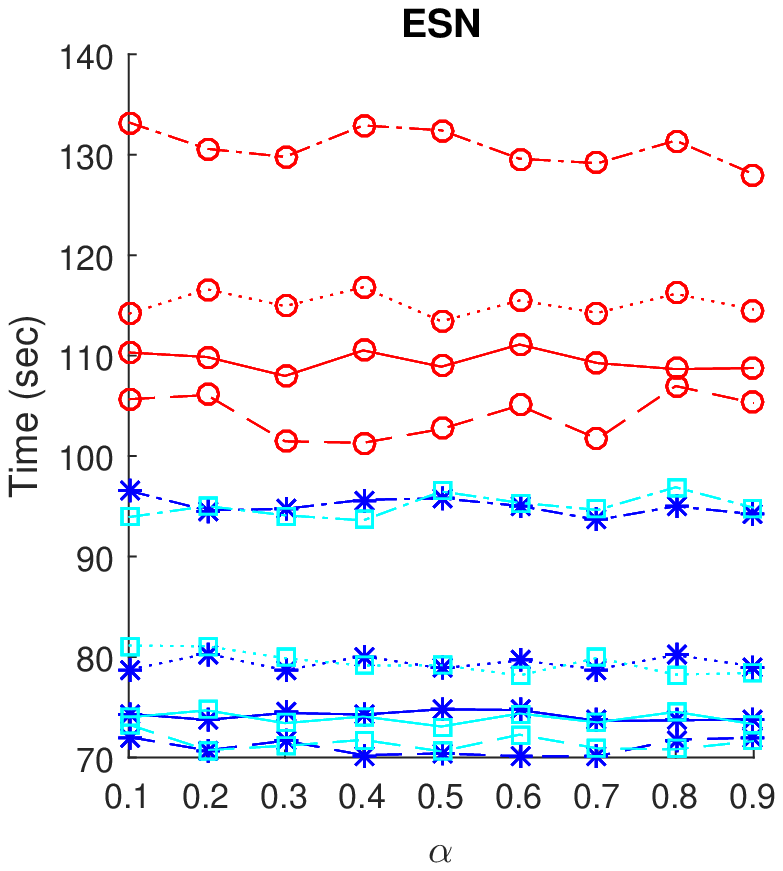}
	\includegraphics[width=0.45\textwidth]{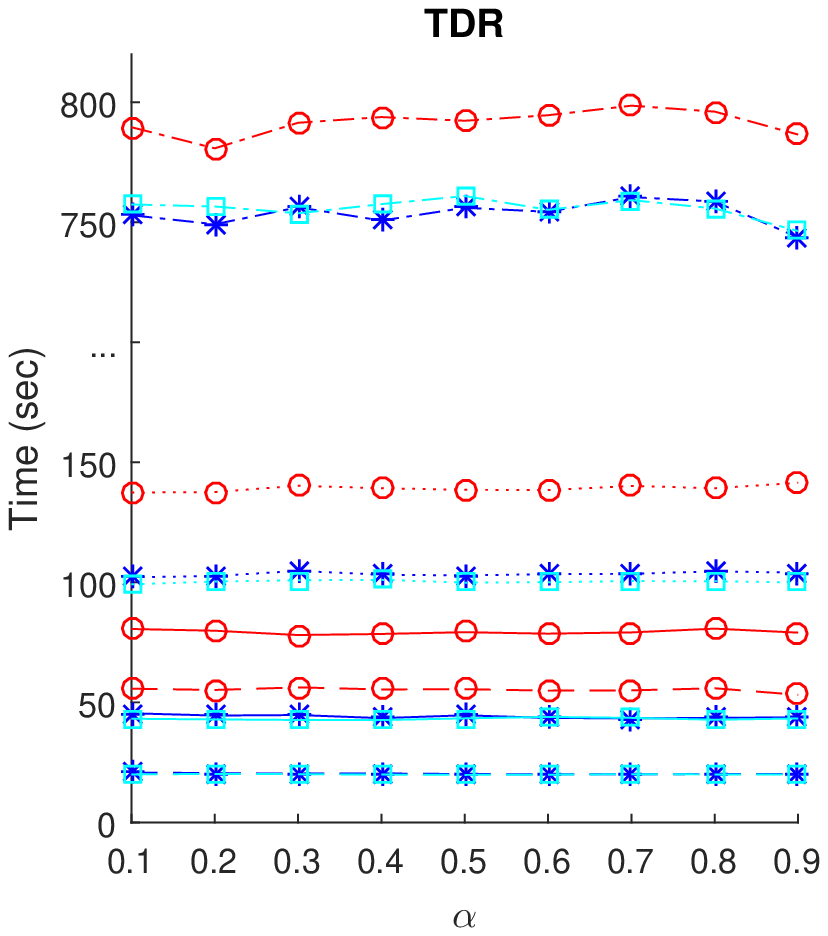}

	\includegraphics[height=1.25in]{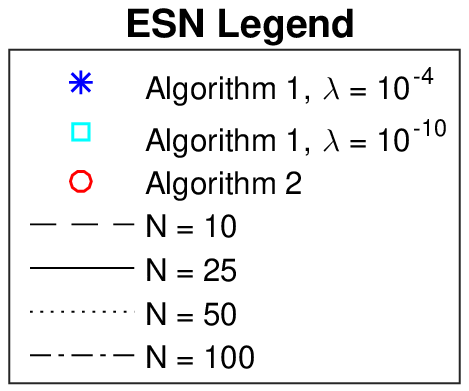}\hspace{1.5 in}
	\includegraphics[height=1.25in]{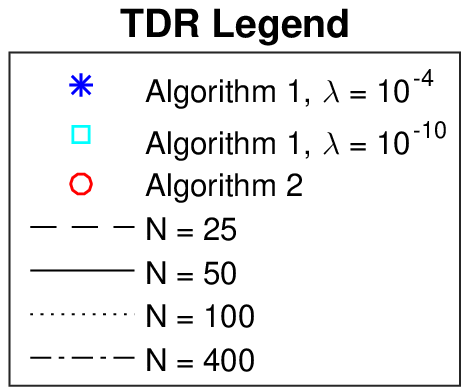}
\caption{A comparison of the performance of the proposed method using Algorithm~\ref{alg:cluster} and the method of trained linear output weights using Algorithm~\ref{alg:Wout t} on ESN and TDR reservoirs for various parameters $\alpha$ and reservoir size $N$.  The top two figures present the classification accuracy on the test set, and the bottom two figures present the total time required in seconds to classify the entire test set of 7000 images.  The plots with red '{\color{red}$\circ$}' denote results from Algorithm~\ref{alg:cluster}, and the plots with blue `{\color{blue}$*$}' or cyan `{\color{cyan}$\square$}' denote results from Algorithm~\ref{alg:Wout t} for different regularization parameters.}
\label{fig:4}
\end{figure}

\subsection{Results}
The results of these simulations are presented in Figures~\ref{fig:3}-\ref{fig:5}.

Figure~\ref{fig:3} plots the maximum separation ratio from Equation~\eqref{eq:6} attained by ESN and TDR reservoirs for both Algorithms for the selections of parameters $\alpha$ and $N$.  Notice that the reservoirs do not separate this data particularly well, but the norms of reservoir responses that are used in Algorithm~\ref{alg:cluster} tend to be slightly better separated than the vector responses used in Algorithm~\ref{alg:Wout t}.

The top two plots of Figure~\ref{fig:4} show the average classification accuracy, and the bottom two plots give the time required to classify all 7000 images in the test set using both reservoir types with both Algorithm~\ref{alg:Wout t} and Algorithm~\ref{alg:cluster} for the parameters $\alpha, N$ and $\lambda$.  The results for the clustering approach presented in Algorithm~\ref{alg:cluster} are denoted by red `{\color{red}$\circ$}'.  The results for the trained output weights using Algorithm~\ref{alg:Wout t} use blue `{\color{blue}$*$}' (for $\lambda=10^{-4}$) or cyan {\color{cyan}$\square$}' (for $\lambda=10^{-10}$).

The clustering approach always achieves a higher classification accuracy than the trained linear output weights, but takes only about 35-40 seconds longer to classify all 7000 images.    Notice the clustering approach is fairly robust to the choice of reservoir and parameters $N$ and $\alpha$.  The trained linear output weights are more sensitive to $N$ and $\alpha$, and are inversely related to the separation ratio given in Figure~\ref{fig:3}.  

The computational complexity of the two Algorithms can be seen in the `Time' plots of Figure~\ref{fig:4}.  For ESNs, both Algorithms have a quadratic dependency on $N$, but Algorithm~\ref{alg:cluster} takes a bit longer also having a quadratic dependency on $|\Omega| = T$.  For TDRs, the linear dependence on $N$ for both Algorithms is evident in the plot.  The TDR has a longer runtime than the ESN since a mask is used with the TDR, increasing $T$ by a factor of $N-1$.

The clustering approach in Algorithm~\ref{alg:cluster} was also applied to the raw input dataset without using a reservoir.  Over 100 trials, the average accuracy of the clustering method applied to the raw input data is 95.27\%, which is smaller than the average accuracy attained by Algorithm~\ref{alg:cluster} using an ESN or TDR.  This suggests that the clustering method is well-suited to this problem, but processing the data in a reservoir improves accuracy for most parameter choices since the reservoir preserves the spatial correlations well.

\begin{figure}
	\centering
	\includegraphics[width=0.48\textwidth]{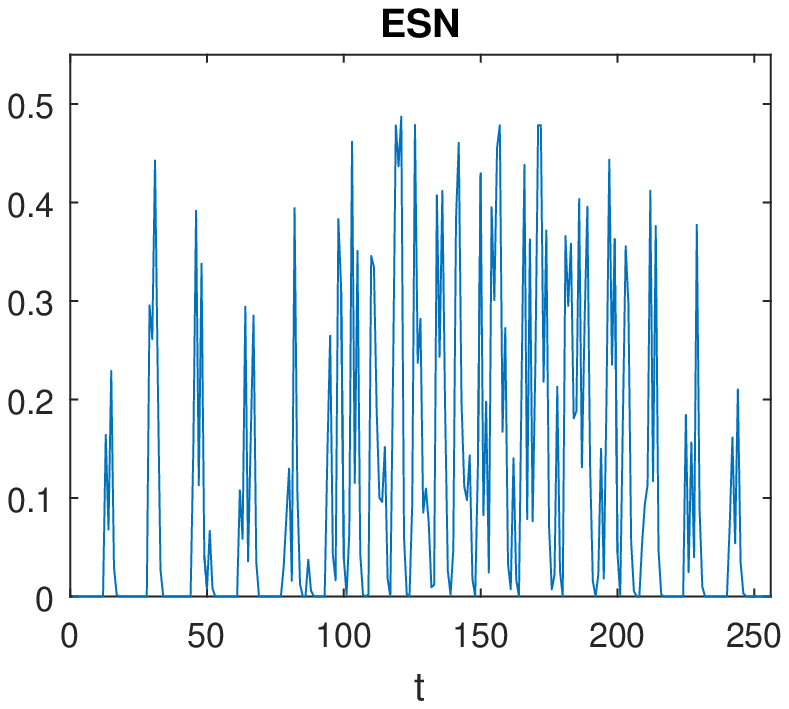}
	\includegraphics[width=0.48\textwidth]{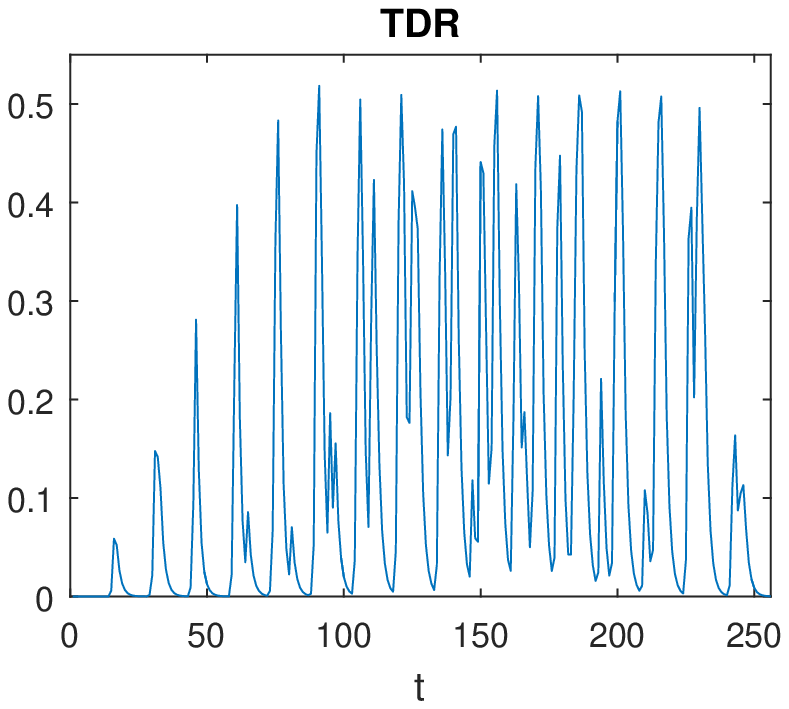}
	\caption{Ratios of the inequalities from Theorems~\ref{th:1} and~\ref{th:2} for randomly selected $i$ and $j$, plotted against values of $t$, with $N = 100$ and $\alpha = 0.5$.}
	\label{fig:5}
\end{figure}

Figure~\ref{fig:5} displays the inequalities presented in Theorems~\ref{th:1} and~\ref{th:2}, measuring the discrepancy of reservoir activations at time $t$ for similar inputs.  The two input signals were randomly selected from the class of `3s'.  The values shown in the figure are are the found by dividing out the right hand side of the inequality, giving 
\[ \varepsilon_{i,j,t} / \left( L \overline{\delta}_{i,j} \|W_\text{in}\|  \frac{1-(L\rho(W_\text{res}))^{t+1}}{1-L\rho(W_\text{res})} \right)\]
in the left image, and
\[ \varepsilon_{i,j,t} / \left( \alpha \overline{\delta}_{i,j} L \sqrt{N}\;  \frac{ 1 - \left(\beta L\right)^{\lfloor t/N\rfloor + 1}}{1 - \beta L} \right) \]
in the right image, both plotted against $t$.  The inequalities in the theorems are clearly satisfied since they are well below 1, however the upper limits could be further refined in future research.

\section{Conclusion}
This work theoretically and experimentally explored a method to classify spatiotemporal patterns using the principal components of norms of reservoir states on a training set.  The proposed method was compared to the traditional method using trained linear output weights for two types of reservoir topologies using several parameter selections.  In the numerical experiments, the proposed method achieved better classification accuracy on the test set, but took a bit longer to complete computations.  The proposed method loses some information since it considers norms of reservoir state vectors, but this leads to more robustness with respect to reservoir type and size, as well as parameter choice.

A basic implementation of both methods was used so the fundamental principles could be compared.  More sophisticated implementations could be used in future work, and may improve speed and accuracy for both methods.  These adaptations could include selecting better training sets, introducing subclasses to reduce intra-class variation and improve class separation, using optimally designed masks for TDRs~\cite{appletantoptimized}, refining the reservoir connections and weights~\cite{jaegerleaky, norton}, improving selection of parameters (spectral radius, reservoir size, feedback strength, regularization parameter) and subsequent solving of trained output weights.

\subsection*{Acknowledgements}
\noindent This research was supported by Air Force Office of Scientific Research [LRIR:15RICOR122]. \\

\noindent Any opinions, findings and conclusions or recommendations expressed in this material are those of the author and do not necessarily reflect the view of the United States Air Force.

\bibliographystyle{elsarticle-num}
\bibliography{mybibfile}

\end{document}